\documentclass{article}

\usepackage{microtype}
\usepackage{graphicx}
\usepackage{subfigure}
\usepackage{booktabs} 

\usepackage{hyperref}
\usepackage{amsmath,amssymb,enumerate,amsthm}

\usepackage[accepted]{icml2021}

\icmltitlerunning{Privacy Amplification by Bernoulli Sampling}

\def\calA{\mathcal{A}}

\def\calC{\mathcal{C}}

\def\barp{\overline{p}}

\def\barx{\overline{x}}
\def\bary{\overline{y}}
\def\R{\mathbb{R}}
\def\post{\textnormal{Post}}
\def\Post{\textnormal{Post}_{\calA, \alpha}(\varepsilon)}
\def\Postk{\textnormal{Post}_{\calA, \alpha,k}(\varepsilon)}
\def\Bern{\textnormal{\texttt{Bern}}}
\def\maj{\textnormal{Maj}}

\def\renyi{R_\alpha}

\def\indicator{\textbf{1}}

\newtheorem{theorem}{Theorem}
\newtheorem{lemma}{Lemma}
\newtheorem{defn}{Definition}
\def\supp{\textnormal{supp}}

\begin{document}

\twocolumn[
  \icmltitle{Privacy Amplification by Bernoulli Sampling}



\icmlsetsymbol{equal}{*}

\begin{icmlauthorlist}
\icmlauthor{Jacob Imola}{ed}
\icmlauthor{Kamalika Chaudhuri}{ed}
\end{icmlauthorlist}

\icmlaffiliation{ed}{UC San Diego}

\icmlcorrespondingauthor{Jacob Imola}{jimola@eng.ucsd.edu}

\icmlkeywords{Machine Learning, ICML}

\vskip 0.3in
]



\printAffiliationsAndNotice{\icmlEqualContribution} 

\begin{abstract}
	Balancing privacy and accuracy is a major challenge in designing
  differentially private machine learning algorithms.  One way to improve this
  tradeoff for free is to leverage the noise in common data operations that
  already use randomness. Such operations include noisy SGD and data subsampling. The additional
  noise in these operations may amplify the privacy guarantee of the overall
  algorithm, a phenomenon known as {\em{privacy amplification}}. In this paper, we 
  analyze the privacy amplification of sampling from a multidimensional
  Bernoulli distribution family given the parameter from a private
  algorithm. This setup has applications to Bayesian inference and to data
  compression. We
  provide an algorithm to compute the amplification factor, and we
  establish upper and lower bounds on this factor. 
\end{abstract}
\section{Introduction}

Differential privacy~\cite{Dwork:2014} has emerged as the gold standard for
privacy in machine learning. Informally, differential privacy ensures 
that the participation of a single person in a dataset does not change the
probability of any outcome by much. This is achieved by adding noise to a
function, such as a classifier or a generative model, computed on the
sensitive data. In order to avoid adding noise to the data as much as possible,
a trove of techniques have been developed to analyze the privacy guarantees of
combinations of black-box private algorithms. One important example of this is
privacy amplification.

Privacy amplification~\cite{Feldman:2018} considers black-box private machine learning 
algorithms which are modified by operations already involving noise. 
This is quite common in private machine learning, such as when a point is
privately selected for a gradient descent step and then future steps not
involving the point are taken, or when a dataset is randomly subsampled before
being given to a private algorithm. Because these operations involve randomness
and differential privacy guarantees tend to get stronger with more randomness,
privacy amplification occurs in the modified algorithm.
In the first example, the randomness in future
gradient steps gives us privacy amplification by iteration~\cite{Feldman:2018,
Balle:2019}, and in the second,
the randomness in sampling the data gives us privacy amplification by
subsampling~\cite{Chaudhuri:2006, 
  Smith:2008, Li:2011, Beimel:2014, Beimel:2014-2, Bun:2015, Balle:2018, 
Wang:2019, Wang:2019poisson, Mironov:2019}.

In this work, we study privacy amplification for an operation not yet
considered in the literature: sampling from a distribution family given the parameter
from a black-box private algorithm.
This operation has direct applications to Bayesian inference: suppose a private
algorithm estimates the posterior distribution, but only samples from the
posterior are released. In addition, it has applications to sparsifying the
weights of a neural network to improve computation speed~\cite{Aji:2017}. In this case,
one could amplify the privacy of a neural network training algorithm if the
neural network were sparsified with a Bernoulli sampling process.

We consider sampling from the multidimensional, independent Bernoulli distribution family
parametrized by $\theta \in [0,1]^d$. This distribution family admits relatively
simple analysis, and is rich enough 
to still be applicable to posterior sampling and to neural network sparsifying.
The question we ask is the following:
If $\theta$ is given by a black-box, differentially-private algorithm, then how does the
privacy guarantee amplify when $k$ samples drawn from $\text{Bernoulli}(\theta)$,
rather than $\theta$, are released?

We provide three answers to this question: First, in Section~\ref{sec:single},
we detail an algorithm for computing the exact privacy amplification. Unfortunately,
this algorithm has running time exponential in $d$, and thus cannot be used for
larger $d$.
Second, in Section~\ref{sec:single-bounds}, we provide closed-form upper 
and lower bounds on the privacy amplification. The bounds can approximate the 
amplification without having to run the algorithm. We provide two cases where
the bounds are a good approximation, and one of the cases is a negative result
that when $d$ is large, then little amplification occurs. Third, we plot the
actual amplification alongside our upper and lower bounds in Section~\ref{sec:valid}.
The plots support our finding that the bounds are good approximations.

\section{Preliminaries}
Our privacy amplification results apply to differentially private algorithms
processed by a Bernoulli sampling operation. In this section, we introduce the
sampling operation, R\'enyi divergence, and R\'enyi differential privacy.
We conclude with a formal problem setup using these
definitions.

\paragraph{Bernoulli Sampling Process}
We consider private algorithms which have range $\Theta \subseteq [0,1]^d$. 
We restrict the range to be in $[0,1]^d$ because we will apply a Bernoulli
post-sampling operation to the output $\theta \in \Theta$. The exact operation
is the following:

\begin{defn}(Independent Bernoulli Process)
  Let $B_k(\theta)$ denote the process which, for 
  an input $\theta = (\theta_1, \ldots, \theta_d) \in [0,1]^d$, repeats the
  following $k$ times: flip $d$ coins 
  independently with biases $(\theta_1,\ldots, \theta_d)$ and return their
  output. 
\end{defn}

Each Bernoulli distribution is independent among its $d$ coordinates which is
less general than an arbitrary distribution on $[0,1]^d$, but still complex
enough to express common post-sampling problems (see Introduction).

\paragraph{R\'enyi Divergence and R\'enyi DP}
R\'enyi differential privacy uses the familiar R\'enyi
divergence $R_\alpha(P,Q)$~\cite{Mironov:2017}. We defer a definition of this to
Appendix~\ref{app:omit-defn}, but we find it convenient to introduce the
following function $r_\alpha(p)$, a special case of the R\'enyi Divergence.
\begin{defn} ($r_\alpha(p)$):
  Let $\{x_1, x_2\} \subseteq \Theta$. Let $P,Q$ be random variables with
  support on $\{x_1, x_2\}$. Suppose $p = \Pr[P = x_1]$ and 
  $1-p = \Pr[Q = x_1]$. Define the binary, symmetric R\'enyi divergence function 
  $r_\alpha(p)$ as follows:
  \begin{multline*}
    r_\alpha(p) = R_\alpha(P,Q) \\ = \frac{1}{\alpha-1} \log \left(
    p^{\alpha}(1-p)^{1-\alpha} + (1-p)^{\alpha} p^{1-\alpha} \right)
  \end{multline*}
\end{defn}

Due to space limitations, we assume the reader is familiar with R\'enyi
differential privacy. For a full discussion, see Appendix~
\ref{app:omit-defn}. The following is a slightly non-standard way to view
R\'enyi DP which will be useful for our purposes.

\begin{defn}\label{def:eps-func} ($\varepsilon_A(\alpha)$): 
  Let $\varepsilon_A(\alpha)$ be the smallest $\varepsilon$ for which $A$
  satisfies $(\alpha, \varepsilon)$-RDP. This quantity is given by
  \[
    \varepsilon_A(\alpha) = \sup_{D \sim D'} \renyi(A(D), A(D'))
  \]
\end{defn}
Notice that $\varepsilon_A(\alpha) \leq \varepsilon$ if and
only if $A$ satisfies $(\varepsilon, \alpha)$-RDP.

R\'enyi differential privacy enjoys an important property that serves as a
trivial lower bound for privacy amplification. This property,
{\bf Post-processing invariance}~\cite{Mironov:2017}, states that if we apply a
randomized function $M : \Theta \rightarrow \Theta'$ to the output of algorithm
$A$, then
$\varepsilon_{MA}(\alpha) \leq \varepsilon_{A}(\alpha)$.

\paragraph{Problem Setup}
In deployments of privacy, data analysts have global constants
$\varepsilon, \alpha$ and must release a private algorithm $A(D)$ on a database
$D$ such that $\varepsilon_A(\alpha) \leq \varepsilon$. An RDP algorithm $A$ is
always parametrized in terms of $\varepsilon$ and $\alpha$; we find it useful to
think of $A$ as a family of private algorithms $\calA$ formed by all possible
instantiations of $A$ with certain $\alpha,\varepsilon$.

For an $A \in \calA$, we consider what 
happens when instead of releasing $A(D)$, the analyst takes an output $\theta =
A(D)$ and releases $B_k(\theta)$ instead of $\theta$.
The amplification at level $\varepsilon, \alpha$ is the worst-case RDP 
guarantee of the algorithm
$B_k(A(\cdot))$ as $A$ ranges over those $A \in \calA$ that satisfy $(\alpha,
\varepsilon)$-RDP. Formally,

\begin{defn} ($\Postk$):
  Let $B_k(\theta)$ be the operation that releases $k$ independent samples from
  $\Bern(\theta)$ for $\theta \in [0,1]^d$.
  Given a family of algorithms $\calA$ such that each algorithm has range on
  $\Theta = [0,1]^d$, the
  {\bf Bernoulli post-processing amplification function $\Postk$}
  is given by
  \begin{equation}\label{eq:post}
    \Postk = \sup_{A \in \calA, \varepsilon_A(\alpha) \leq
    \varepsilon} \varepsilon_{B_k(A)}(\alpha)
  \end{equation}
  We abbreviate $\post_{\calA, \alpha, 1}(\varepsilon)$ to $\Post$.
\end{defn}

One can see that $\Postk$ captures exactly what we mean by privacy amplification
over a class $\calA$. At a given privacy level $\varepsilon$, $\Postk$ returns
the worst-case privacy improvement over all $A \in \calA$.
Using post-processing invariance, we can derive that $\Postk \leq \varepsilon$.
The next sections explore how far away $\Postk$ is from $\varepsilon$ while
making minimal assumptions about $\calA$.

\section{Single Sample Amplification}
\label{sec:single}

In this section, we present our algorithm for computing $\Post$; i.e. $k=1$. The algorithm
greatly improves a naive search, but still has running time $2^{O(d)}$. Before
introducing our result, we must make an a light assumption about $\calA$ which
we now illustrate.

Suppose $\calA$ consists of all algorithms ranging over $\Theta = [0,1]^d$.
Then, some algorithms in $\calA$ trivially have no amplification. Specifically,
if $A \in \calA$ outputs just values in $\{0,1\}^d$, then it is not hard to see
that releasing a sample from $\Bern(\theta)$ for any $\theta \in \{0,1\}^d$ gives
$\theta$ back. This implies $\varepsilon_{B_1(A)}(\alpha) =
\varepsilon_A(\alpha)$.
At the very least, we must exclude these algorithms from $\calA$. To keep the
assumptions as light as possible, we assume $\Theta \subseteq [c,1-c]^d$
for a small constant $c < \frac{1}{2}$. 

\paragraph{Amplification Under Restricted $\Theta$}
Naively, to compute $\Post$, we could 
search over all distributions $P,Q$ on $\Theta$ subject to the 
constraints $\renyi(P,Q) \leq \varepsilon$ and 
$\renyi(Q,P) \leq \varepsilon$ to find the maximal value
$\overline{\varepsilon}$ of
$\renyi(B_1(P), B_1(Q))$. This would give the worst-case values of $P,Q$, and
hence we would have $\Post = \overline{\varepsilon}$.
We could search this space by discretizing
$\Theta$ into bins of width $\delta$ and then using $\Omega((\frac{1-2c}{\delta})^d)$ 
real-valued variables to represent the mass of $P,Q$ in each bin.

To greatly reduce this search space, we
can use a convexity argument to show that $\renyi(B_1(P), B_1(Q))$ is maximized
when $P,Q$ have support on $\{c,1-c\}^d$. We call these distributions the
$c$-corner distributions, and denote them by $\calC_d$.

The following theorem establishes that the
$P,Q$ which maximize $\renyi(B_1(P), B_1(Q))$ subject to
$\renyi(P,Q),\renyi(Q,P) \leq \varepsilon$ are a subset of the $c$-corner
distributions.

\begin{theorem}
  \label{thm:multi-bounded-amp}
  Let $\calA$ be the set of all algorithms with output on $[c, 1-c]^d$. Then, for
  all $\alpha > 1$, $\varepsilon \geq 0$,
  \begin{equation*}
    \Post =
    \sup_{\substack{ P,Q \in \calC_d \\ \renyi(P,Q),\renyi(Q,P) \leq \varepsilon
    }} \renyi(B_1(P), B_1(Q))
  \end{equation*}
\end{theorem}

This theorem gives the exact value of $\Post$; it is not an upper bound.
However, the $c$-corner distributions are a partial characterization of the
worst-case distributions since not all will have the worst-case value of
$\renyi(B_1(P), B_1(Q))$.

Computing the optimization problem of Theorem~\ref{thm:multi-bounded-amp} 
requires searching
over a much smaller space than the naive method.
We can write the problem using $2^{O(d)}$ variables, which is much better than
the naive method. Finally, each constraint and the objective are convex, so we
can use an iterative or convex method to solve the optimization problem for
smaller $d$. We illustrate this procedure in Algorithm~\ref{alg:amp-opti}. The proofs of 
correctness appear in Appendix~\ref{sec:app-algo}.

\section{Single Sample Amplification Lower and Upper Bounds}
\label{sec:single-bounds}
While Theorem~\ref{thm:multi-bounded-amp} exactly describes $\Post$, its exponential
running time in $d$ prevents it from being practical for higher values of $d$.
In this section, we sacrifice some precision and derive 
a lower and an upper bound on $\Post$ that have closed forms. To get a sense of
when the bounds are tight, we prove that the upper bound is close to $\Post$ for
two cases: when $\varepsilon$ is high and when $d$ is moderately high 
and $c,\alpha$, and $\varepsilon$ have reasonably small values. The second
result is somewhat negative as it implies $\Post \approx \varepsilon$,
meaning not much amplification holds. Nonetheless, the result improves our
understanding of what $\Post$ looks like for these cases.

\subsection{Upper and Lower Bounds for $\Post$}

Our upper bound is a minimum of two upper bounds---the first, following from
post-processing invariance~\cite{Mironov:2017}, is that $\Post \leq \varepsilon$. 
The second is an asymptotic value of~\ref{thm:multi-bounded-amp}
given by $R(B_1(\{c\}^d, B_1(\{1-c\}^d) = d
r_\alpha(c)$ (the proof of this appears in Appendix~\ref{app:proofs},
Theorem~\ref{thm:amp-asymptote-multi}). Thus, we have the following upper
bound:

\begin{equation}\label{eq:amp-ub}
  \Post \leq \min \{ \varepsilon, dr_\alpha(c) \}
\end{equation}
Notice the second part of the upper bound is linear in the dimension $d$. This
is due to the additivity of R\'enyi divergence on independent distributions of
dimension $d$.

For the lower bound, we simply plug two distributions $P,Q \in \calC^d$ into
Theorem~\ref{thm:multi-bounded-amp}, where
\begin{align}
  P &= p \indicator[X=\{c\}^d] + (1-p) \indicator[X = \{1-c\}^d] 
  \label{eq:worst-case1}\\
  Q &= (1-p) \indicator[X=\{c\}^d] + p \indicator[X = \{1-c\}^d]
  \label{eq:worst-case2}
\end{align}
and $p\leq\frac{1}{2}$. Notice $\renyi(P,Q) = \renyi(Q,P) = r_\alpha(p)$.
This results in the lower bound:
\begin{equation}\label{eq:amp-lb}
  \renyi(B_1(P),B_1(Q)) \leq \post_{\calA, \alpha}(r_\alpha(p))
\end{equation}

Unlike the situation in Theorem~\ref{thm:multi-bounded-amp}, there is a way to 
compute $\renyi(B_1(P), B_1(Q))$ efficiently
in all parameters including $d$ (see Appendix~\ref{app:proofs}, Theorem~\ref{thm:amp-lb}).
While we can't prove it here, we conjecture that~\eqref{eq:amp-lb} is actually
an equality for all $p \in (0,\frac{1}{2}]$. Proving this would give a full
characterization of distributions $P,Q$ which provide the worst amplification
(see remarks after Theorem~\ref{thm:multi-bounded-amp} ).

\subsection{Cases Where Bounds Are Close}
\label{sec:ub-regime}
We discuss two cases where the upper and lower bounds of the last section
are close, and thus they give a good approximation for $\Post$. First, one can verify
that the lower bound~\eqref{eq:amp-lb} actually approaches the
asymptote in the upper bound~\eqref{eq:amp-ub}. This means they grow very close
for large $\varepsilon$. Second, when
when the dimension $d$ is large, then points $\{c\}^d$ and $\{1-c\}^d$ barely mix under 
the Bernoulli
sampling operation. Thus, if $P,Q$ have support on just these two points, then 
$B_1(P)$ and $B_1(Q)$ will be 
about as hard to distinguish as $P$ and $Q$. This implies $\Post
\approx \varepsilon$. This idea gives us the following:

\begin{theorem}
  \label{thm:multi-lower-bound}
  Let $p \in (0, \frac{1}{2}]$, let $P = p \textbf{1}[X = \{c\}^d] +
  (1-p)\textbf{1}[X = \{1-c\}^d]$ and $Q$ be the same distribution with $p$
  replaced by $1-p$. Let $K = e^{-2(1/2-c)^2d}$.
  If $p + K \leq \frac{1}{2}$, then $\renyi(B_1(P), B_1(Q)) \geq r_\alpha(p
  + K)$.
\end{theorem}

For all $p \in (0, \frac{1}{2}]$, if $K$ satisfies the theorem statement,
then Theorem~\ref{thm:multi-lower-bound}, combined with~\eqref{eq:amp-ub}
and~\eqref{eq:amp-lb}, gives
\[
  r_\alpha(p+K) \leq R_\alpha(B_1(P), B_1(Q)) \leq
  \post_{\calA,\alpha}(r_\alpha(p)) \leq r_\alpha(p).
\]
Thus, the upper and lower bounds on $\Post$ are close whenever 
$K$ is much smaller than $p$, as this will imply $r_\alpha(p+K)$ and
$r_\alpha(p)$ are close. Since $K$ decays exponentially with
$2(\frac{1}{2}-c)^2d$, for moderate values of $d$, $K$ will be extremely small
compared to $p$. The exceptions to this rule are when $p$ is very small (meaning
$r_\alpha(p)$ is large) or when $c$ is very close to $\frac{1}{2}$. In summary,
we've shown that $\Post \approx \varepsilon$ when (1) $d$ is at least moderate
in value, (2) $c$ is not too far from $\frac{1}{2}$, and (3) $\varepsilon$ is
not too large. In Section~\ref{sec:valid}, we empirically validate this rule.

\section{Validation}\label{sec:valid}

Recall for large $d,k$, our understanding of $\Post$ is limited to our upper
and lower bounds~\eqref{eq:amp-ub} and~\eqref{eq:amp-lb} due 
to the exponential running time of our algorithms.
Theorem~\ref{thm:multi-lower-bound} led us to the conclusion that when $d$ is
moderate in value and $c,\alpha$ are reasonable, then $\Post \approx
\varepsilon$. In this Section, we validate this last statement and the closeness 
of our bounds empirically.

{\bf Setup.} For $k=1$, we plot the following functions of $\varepsilon$: $\Post$, 
the lower bound~\eqref{eq:amp-lb} (labeled LB) and
both parts of the upper bound $\min\{\varepsilon,
dr_\alpha(c)\}$~\eqref{eq:amp-ub} (labeled PPI
and Asympt.). We produce a
plot for $d \in \{1,2,3,5,15\}$, $c \in \{0.01, 0.10, 0.30\}$, and $\alpha \in
\{5, 50\}$. For $d > 3$, we do not plot $\Post$ due to the runtime of
Algorithm~\ref{alg:amp-opti}. The plots for $\alpha = 50$, appear in
Figure~\ref{fig:dim-plot}. The plots for $\alpha =
5$ are in Appendix~\ref{app:graphs}.
For $k > 1$, we plot $\Post$, the lower bound~\eqref{eq:amp-lb} (labeled
LB), and both parts of the upper bound $\min\{ \varepsilon,
dkr_\alpha(c)\}$~\eqref{eq:amp-lb} (labeled PPI and Asympt.).
We fix $c = 0.1$ and $\alpha = 50$ 
and vary $k \in \{1,2,4\}$ and $d \in \{1,3,5\}$.
These plots appear in Appendix~\ref{app:graphs}. Due to numerical
instability, some of the plotted values do not extend across the whole 
$\varepsilon$ domain.

{\bf Results.}
The plots for $\alpha = 5$ and $\alpha = 50$ look almost identical, so the
following conclusions hold for both values. For all plots, there are three
notable Regimes for $\varepsilon$: In Regime I, $\Post$ is close to
$\varepsilon$. As predicted by Theorem~\ref{thm:multi-lower-bound}, Regime I
lasts from a very small value of $\varepsilon$ to a moderate value, and lasts
longer with smaller $c$ or larger $d$. Often
Regime I lasts until $\varepsilon$ is quite large. Only when $(d,k,c) \in
\{(\leq 5, 1, 0.3), (\leq 2, 1, 0.1), (1, 1, 0.01), (1, 2, 0.1)\}$ does Regime I 
end before
$\varepsilon = 5$ which we consider to be an upper limit for a reasonable 
$\varepsilon$. That Regime I generally lasts until high values of $\varepsilon$
supports our claim that $\Post \approx \varepsilon$ for many reasonable values
of $c,\alpha$, and $d$.

In Regime II, $\Post$ flattens out. The gap
between the upper and lower bounds is highest, but still never higher than about
$1.5$. The gap between lower and upper bound is never very high, and this
supports that these bounds are a good approximation to $\Post$.

In Regime III, $\Post$ has converged to its asymptote of
$dkr_\alpha(c)$~\eqref{eq:amp-ub}; Regime III starts soon after Regime I ends
due to Regime II being relatively short. In this Regime, the upper bound is very
close to $\Post$.

Finally, $\Post$ and the lower bound~\eqref{eq:amp-lb} are always
visually equal on all our plots. This suggests that $\Post$ is actually equal
to ~\eqref{eq:amp-lb}, though we do not yet have proof. If true, it would
give an efficient method for computing $\Post$.

\begin{figure}[ht]
    \centering
    \includegraphics[scale=0.7]{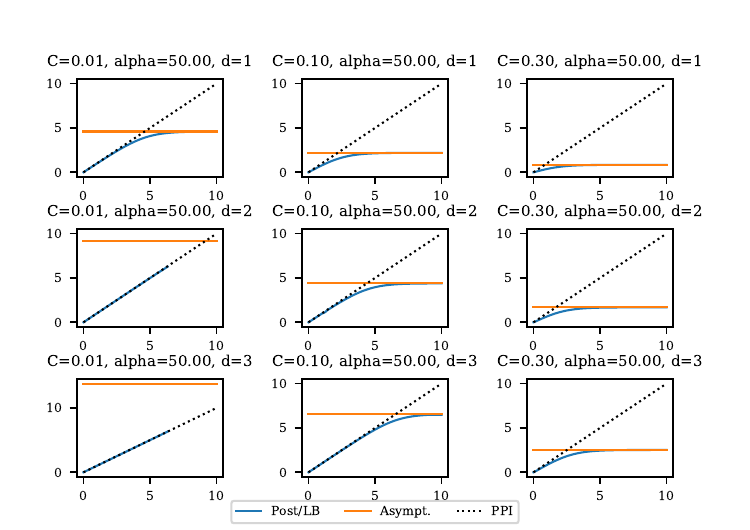}

    \includegraphics[scale=0.7]{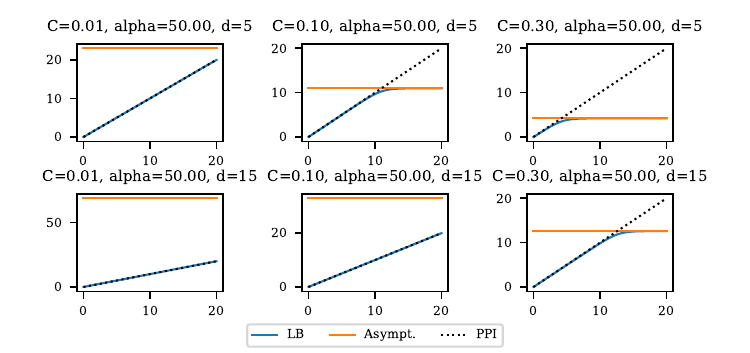}
  \vskip -0.1in
  \caption{Comparison of $\Post$, lower bound~\eqref{eq:amp-lb} (LB),
  upper bound~\eqref{eq:amp-ub} (PPI and Asympt.) for $\alpha=50$.}
    \label{fig:dim-plot}

  \vskip -0.1in
  \end{figure}

\section{Conclusion}
We initiate the study of privacy amplification for a new operation: sampling
from a Bernoulli distribution given the parameter from a private algorithm.
Given the dimension of the distribution, privacy level, and number of samples,
we characterize the amount of amplification that occurs. 
For future work, one might consider a different
sampling procedure or making a stronger assumption on the private algorithm to
enable the possibility of more amplification. Proving that privacy amplification
occurs would allow one to strengthen many privacy guarantees for free.

\bibliography{ms}
\bibliographystyle{icml2021}
\appendix
\section*{Overview}
Appendix~\ref{app:omit-defn} contains details on our preliminary
definitions including R\'enyi differential privacy.
In Appendix~\ref{sec:multi-sample}, we describe the privacy amplification
algorithm for general $k$, and produce lower bounds and upper bounds similar to
the $k=1$ case.
In Appendix~\ref{app:graphs}, we include plots that were omitted from
Section~\ref{sec:valid}. Appendix~\ref{sec:related} contains related works.
Finally, Appendix~\ref{sec:app-algo} contains our algorithms for computing
amplification, and Appendix~\ref{app:proofs} contains all other omitted proofs.

\section{Omitted Definitions}\label{app:omit-defn}

Since we will be working with R\'enyi differential privacy, we introduce the 
R\'enyi divergence and important related concepts.
We assume throughout that $\alpha \in \R$ is a parameter bigger that $1$ and that
$\varepsilon \in \R$ is a parameter bigger than $0$, even if 
we don't state this explicitly. The R\'enyi divergence is defined as:
\begin{defn} ($\renyi(P,Q)$):
For distributions $P,Q$ on $\Theta$, the R\'enyi
divergence of order $\alpha$ between $P$ and $Q$ denoted $\renyi(P,Q)$, is
  $\renyi(P,Q) = \frac{1}{\alpha-1} \log\left( \int_{\Theta} dP^{\alpha} 
  dQ^{1-\alpha} \right)$
\end{defn}

The R\'enyi divergence can be thought of as a ``distance'' between
distributions, though it does not satisfy the properties of a usual metric.
Still, it enjoys many useful properties
such as quasi-convexity, independent composition, and being related to
$f$-divergences~\cite{erven:2012,Cichocki:2010}. One very important property 
is the {\bf data processing inequality}~\cite{erven:2012}.  This property states that if two
distributions $P,Q$ on $\Theta$ are post-processed by a randomized function $M :
\Theta \rightarrow \Theta'$, then
$\renyi(P,Q) \geq \renyi(MP, MQ)$.

\paragraph{R\'enyi Differential Privacy}
R\'enyi differential privacy (RDP)~\cite{Mironov:2017} is an alternative definition 
of differential privacy~\cite{Dwork:2006} that has enjoyed a surge of recent interest
because of tighter analyses for common operations such as
subsampling~\cite{Wang:2019} and composition~\cite{Mironov:2017}. Because of its
improvements over standard differential privacy, we use RDP as our privacy definition of
choice in this paper. RDP is defined using a binary, symmetric relation $\sim$
on databases. For two databases $D,D'$, we say $D \sim D'$ if $D$ and $D'$ differ 
in just one record.

\begin{defn} ($(\alpha, \varepsilon)$-RDP):
  Let $A$ be a randomized algorithm with range $\Theta$, and let $\alpha \in \R$ be a
  parameter bigger than $1$. $A$
  satisfies $(\alpha, \varepsilon)$-RDP if $\sup_{D \sim D'} 
  \renyi(A(D), A(D')) \leq \varepsilon$.
\end{defn}
The goal of RDP is to hide the contribution of a single record on the output of
$A$ by making $A(D)$ and $A(D')$ indistinguishable. The R\'enyi divergence
measures the distinguishability of $A(D)$ and $A(D')$. The parameter 
$\varepsilon$ gives the strength of the RDP guarantee, since it is an upper
bound on the distinguishability level.

One annoyance with the above definition that we will need to correct is that
$\varepsilon$ is an upper bound on the privacy guarantee: if $A$ satisfies
$(\alpha, \varepsilon)$-RDP, then $A$ satisfies $(\alpha, \varepsilon')$-RDP for
all $\varepsilon' > \varepsilon$. We are interested in the tightest possible
privacy guarantee for the post-sampling operation, and thus we introduce a
function that gives the smallest $\varepsilon$ such that an algorithm $A$
satisfies $(\alpha, \epsilon)$-RDP. This leads us to
Definition~\ref{def:eps-func}.

\section{The Amplification Picture for Multiple
Samples}\label{sec:multi-sample}
We follow a similar setup to Sections~\ref{sec:single}
and~\ref{sec:single-bounds} but analyze $\Postk$ for a general $k$.
While on the surface this problem resembles $k$-fold composition of
GDP~\cite{Mironov:2017},
this is actually incompatible because the private algorithm is run only once,
and the output is used to release $k$ samples. Thus, we expect the amplification
to be lessened as the $k$ samples paint a better picture of the original
output. In this section,
we give $k$ sample analogues to results in
Sections~\ref{sec:single} and ~\ref{sec:single-bounds} that follow a similar
pattern: we give an algorithm
for computing $\Postk$ exactly. Then, we derive upper and
lower bounds on $\Postk$ that are efficiently computable. These bounds are close
when $\varepsilon$ is high, and it is still true that $\Postk \approx
\varepsilon$ when $d$ is moderately high and $c,\alpha$, and $\varepsilon$ are
reasonably small.

While it is more difficult to visualize than when $k=1$, the $c$-corner 
distributions are once again the worst distributions that an algorithm in
$\calA$ can output. The reason is the same if $P,Q$ have mass on a point not in $\{c,
1-c\}^d$, then moving the mass at that point to a point in $\{c, 1-c\}^d$
can only decrease $\renyi(P,Q)$ but increases $\renyi(B_k(P), B_k(Q))$ (see
Section~\ref{sec:single} for more details).
The following theorem formalizes this intuition:
\begin{theorem}
  \label{thm:multi-bounded-amp-ms}
  Let $\calA$ be the set of all algorithms with output on $[c, 1-c]^d$. Then, for
  all $\alpha > 1$, $\varepsilon \geq 0$,
  \begin{equation*}
    \Postk =
    \sup_{\substack{ P,Q \in \calC_d \\ \renyi(P,Q),\renyi(Q,P) \leq \varepsilon 
    }} \renyi(B_k(P), B_k(Q))
  \end{equation*}
\end{theorem}

This implies we can compute $\Post$ by just changing the optimization function
of Algorithm~\ref{alg:amp-opti}. The resulting algorithm is
Algorithm~\ref{alg:amp-opti-general} which appears in
Appendix~\ref{sec:app-algo}. We remark the algorithm
uses $2^{O(d)}$ variables and takes $2^{O(kd)}$
time to compute, so it is not efficient in either $d$ or $k$. Improving matters
somewhat, when $d = 1$, we can
take advantage of independence and symmetry to reduce the runtime to $O(k)$. 
This result is Algorithm~\ref{alg:amp-opti-multi}. The proofs of correctness of
these algorithms lies in Appendix~\ref{sec:app-algo}.

\begin{algorithm}[tb]
  \caption{Algorithm for computing $\Postk$ when $\calA$ consists of algorithms
    with range on $[c, 1-c]$ and $k$ is general. }
  \label{alg:amp-opti-multi}
  \begin{algorithmic}
    \STATE {\bfseries Input:} Constant $c$, privacy parameter $\varepsilon$, order $\alpha$, no.
    samples $k$

  \STATE $Constraint_1 \leftarrow r_\alpha(x,y) $
  \STATE $Constraint_2 \leftarrow r_\alpha(y,x) $
  \FOR{$0 \leq j \leq k$}
    \STATE $\barx_j \leftarrow x c^j + (1-x) (1-c)^{k-j}$
    \STATE $\bary_j \leftarrow y c^j + (1-y) (1-c)^{k-j}$
  \ENDFOR 
  \STATE $Objective \leftarrow
  \sum_{j=0}^k { {k} \choose {j} } (\barx_j / \bary_j)^\alpha \bary_j$
  \STATE $MaxVal \leftarrow maximize(Objective)$ \textbf{ subject to } 
  \STATE $Constraint_1 \leq e^{(\alpha-1)\varepsilon}$
  \STATE $Constraint_2 \leq e^{(\alpha-1)\varepsilon}$
  \STATE $0 \leq x \leq 1, 0 \leq y \leq 1$
  \STATE \textbf{Return } $\frac{1}{\alpha-1}\log(MaxVal)$
\end{algorithmic}
\end{algorithm}

To upper bound $\Postk$, we can use the post-processing upper bound and the
following asymptotic value of $\Postk$, which is similar to that
of~\eqref{eq:amp-ub} (proof appears in
Appendix~\ref{app:proofs}, Theorem~\ref{thm:amp-asymptote-multi}).
\begin{equation}
  \label{eq:amp-ub-multi}
  \Postk \leq \min \{ \varepsilon, dkr_\alpha(c) \}
\end{equation}
For our $k$ sample lower bound,
we use the two distributions $P,Q$
defined in~\eqref{eq:worst-case1}
and~\eqref{eq:worst-case2} for the single sample lower bound~\eqref{eq:amp-lb}.
\begin{equation}\label{eq:amp-lb-multi}
  \renyi(B_k(P), B_k(Q)) \leq \post_{\calA, \alpha, k}(r_\alpha(p))
\end{equation}

In the $k$-sample setting, using these specific $P,Q$ for our lower bound 
is quite nice because there is an efficient way to compute 
$R_\alpha(B_k(P), B_k(Q))$
(Appendix~\ref{app:proofs}, Theorem~\ref{thm:amp-lb}) which is not obvious.
Thus, when $d=1$, it is possible to compute the amplification exactly.
Once again, we conjecture that equality holds
in~\eqref{eq:amp-lb-multi}.
We can compute a range for $\Postk$ efficiently using
~\eqref{eq:amp-ub-multi} and~\eqref{eq:amp-lb-multi}.
While we do not generalize Theorem~\ref{thm:multi-lower-bound}, we note that the
values of $\varepsilon, c,d$, and $\alpha$ that result in
$\Post \approx \varepsilon$ also result in $\Postk \approx \varepsilon$ because 
$\varepsilon \geq \Postk \geq \Post$.

\section{Supplementary Graphs}
\label{app:graphs}

The Supplementary graphs appear in Figures~\ref{fig:large-dim-plot}
and~\ref{fig:multi-sample-plot}. These plots lead to conclusions similar to
those drawn in Section~\ref{sec:valid}.

\begin{figure}[h]
    \centering
    \includegraphics[scale=0.7]{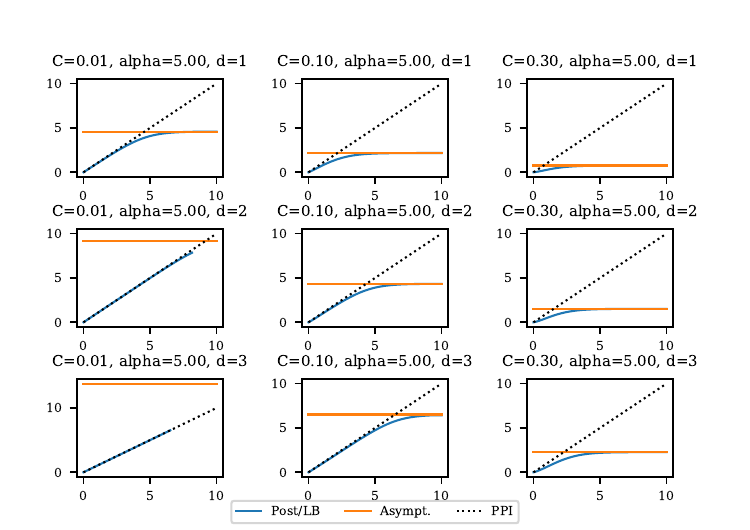}
    \includegraphics[scale=0.7]{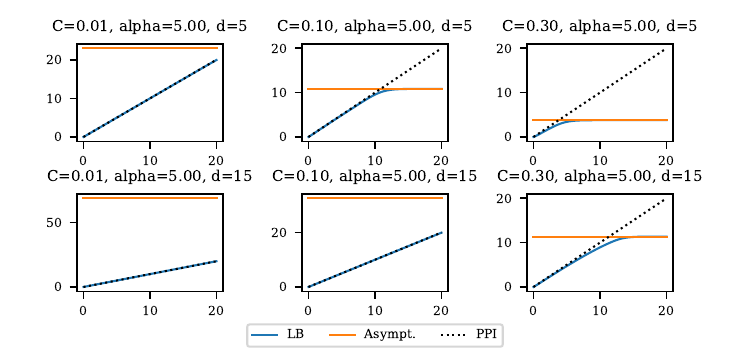}
    \caption{Comparison of $\Post$, lower bound~\eqref{eq:amp-lb} (LB),
    upper bound~\eqref{eq:amp-ub} (PPI and Asympt.) for $\alpha=5$.}
    \label{fig:large-dim-plot}
\end{figure}

\begin{figure}[h]
  \centering
  \includegraphics[scale=0.8]{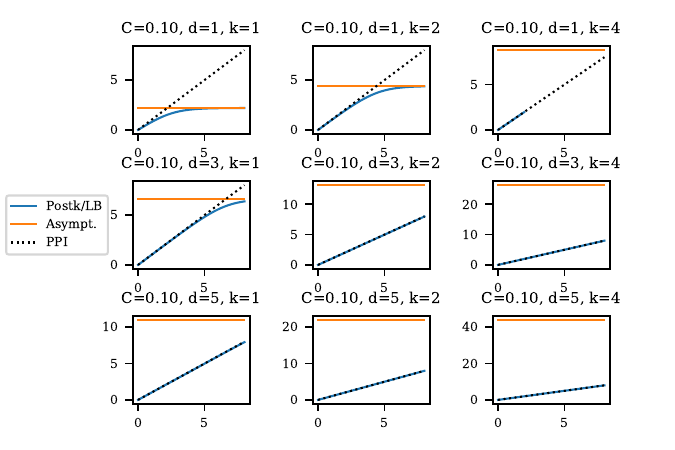}
  \caption{Comparison of $\Postk$, lower bound~\eqref{eq:amp-lb-multi} (LB),
    upper bound~\eqref{eq:amp-ub-multi} (PPI and Asympt.) for $\alpha=5$ and
  $d=1$.}
  \label{fig:multi-sample-plot}
\end{figure}

\section{Related Work}\label{sec:related}

The literature
on privacy amplification can be organized by the type of post-processing
mechanism under consideration. First, amplification by iteration~\cite{Feldman:2018} 
or diffusion~\cite{Balle:2019} occurs through repeated applications of Markov
operators to the output of a private algorithm. While Bernoulli post-processing
is a Markov process, the fact that we focus on this specific problem allows us
to obtain a tighter result.
Second, amplification by subsampling~\cite{Chaudhuri:2006, Smith:2008,
Li:2011, Beimel:2014, Beimel:2014-2, Bun:2015, Balle:2018, Wang:2019, Wang:2019poisson,
Mironov:2019} occurs when a sample from the dataset is input to the private
algorithm, rather than the whole dataset itself. Clearly, this problem
setup is different from ours. Third, privacy amplification by
shuffling~\cite{Erlingsson:2018,Cheu:2018,Balle:2019-shuffle} occurs when
differential privacy guarantees are strengthened under a different privacy
model, the shuffle model. We do not consider the shuffle model here, but it
would be an interesting area of future research.

As Bernoulli post-processing is related to posterior sampling, the Bayesian 
machine learning literature has considered problems related to ours.
Specifically, there are results about the differential privacy guarantees of
posterior sampling under different classes of prior-posterior
distribution families~\cite{rubinstein:2014, Wang:2015, Geumlek:2016,
Geumlek:2017}. Our method takes these results one step further by considering
the amplification when samples from the private posterior, rather
the private posterior itself, are released.

Our work is also related to the literature on strong data processing inequalities
in the information theory community~\cite{Raginsky:2014}. However, the difference 
is that those results apply to $f$-divergences and general processes while
we focus on R\'enyi divergences and the Bernoulli 
post-processing mechanism.

\section{Omitted Algorithms, Proof of Algorithm Correctness}
\label{sec:app-algo}

\begin{algorithm}[tb]
  \caption{Algorithm for computing $\Post$ when $\calA$ consists of algorithms
  with range on $[c, 1-c]^d$. $\Delta(x,y)$ is the Hamming distance between
  $x,y$.}
  \label{alg:amp-opti}
  \begin{algorithmic}
    \STATE {\bfseries Input:} Constant $c$, privacy parameter $\varepsilon$,
    order $\alpha$, dimension $d$
    \STATE $Constraint_1 \leftarrow \sum_{i \in \{0,1\}^d} (x_i / y_i)^\alpha
    y_i $
    \STATE $Constraint_2 \leftarrow \sum_{i \in \{0,1\}^d} (y_i / x_i)^\alpha x_i $
    \FOR{$b \in \{0,1\}^d$}
    \STATE $\barx_{b} \leftarrow \sum_{i \in \{0,1\}^d} x_i
    c^{\Delta(i, b)}(1-c)^{d-\Delta(i, b)}$
    \STATE $\bary_{b} \leftarrow \sum_{i \in \{0,1\}^d} y_i
    c^{\Delta(i, b)}(1-c)^{d-\Delta(i, b)}$
    \ENDFOR
    \STATE $Objective \leftarrow \sum_{b \in \{0,1\}^d} (\barx_{b} / 
      \bary_{b})^\alpha \bary_{b} $
    \STATE $MaxVal \leftarrow maximize(Objective)$ \textbf{ subject to } 
    \STATE \qquad $Constraint_1 \leq e^{(\alpha-1)\varepsilon}$
    \STATE \qquad $Constraint_2 \leq e^{(\alpha-1)\varepsilon}$
    \STATE \qquad $\{x_i \geq 0 : 0 \leq i \leq 2^d\}$ 
    \STATE \qquad $\{y_i \geq 0 : 0 \leq i \leq 2^d\}$
    \STATE \qquad $\sum_{i=0}^{2^d-1} x_i = 1$
    \STATE \qquad $\sum_{i=0}^{2^d-1} y_i = 1)$
    \STATE \textbf{return} $\frac{1}{\alpha-1}\log(MaxVal)$
\end{algorithmic}
\end{algorithm}

\begin{algorithm}[tb]
  \caption{Algorithm for computing $\Postk$ when $\calA$ consists of algorithms
    with range on $[c, 1-c]^d$. $\Delta(x,y)$ is the Hamming distance.}
  \label{alg:amp-opti-general}
\begin{algorithmic}
  \STATE {\bfseries Input: } Constant $c$, privacy parameter $\varepsilon$,
  order $\alpha$, dimension $d$, no. samples $k$.
  \STATE $Constraint_1 \leftarrow \sum_{i \in \{0,1\}^d} (x_i / y_i)^\alpha y_i $
  \STATE $Constraint_2 \leftarrow \sum_{i \in \{0,1\}^d} (y_i / x_i)^\alpha x_i $
  \FOR{$(b_1, \ldots, b_k) \in (\{0,1\}^d)^k$}
    \STATE $\barx_{b_1, \ldots, b_k} \leftarrow \sum_{i \in \{0,1\}^d} x_i
    \prod_{j=1}^k c^{\Delta(i, b_j)}(1-c)^{d-\Delta(i, b_j)}$
    \STATE $\bary_{b_1, \ldots, b_k} \leftarrow \sum_{i \in \{0,1\}^d} y_i
    \prod_{j=1}^k c^{\Delta(i, b_j)}(1-c)^{d-\Delta(i, b_j)}$
  \ENDFOR
  \STATE $Objective \leftarrow
  \sum_{(b_1, \ldots, b_k) \in (\{0,1\}^d)^k} (\barx_{b_1, \ldots, b_k} / 
  \bary_{b_1, \ldots, b_k})^\alpha \bary_{b_1, \ldots, b_k} $
  \STATE
  $MaxVal \leftarrow maximize(Objective)$ {\textbf{ subject to }} 
  \STATE \qquad $Constraint_1 \leq e^{(\alpha-1)\varepsilon}$
  \STATE \qquad $Constraint_2 \leq e^{(\alpha-1)\varepsilon}$
  \STATE \qquad $\{x_i \geq 0 : 0 \leq i \leq 2^d-1\}$
  \STATE \qquad $\{y_i \geq 0 : 0 \leq i \leq 2^d-1\}$
  \STATE \qquad $\sum_{i=0}^{2^d-1} x_i = 1$
  \STATE \qquad $\sum_{i=0}^{2^d-1} y_i = 1$
    \STATE \textbf{return} $\frac{1}{\alpha-1}\log(MaxVal)$
\end{algorithmic}
\end{algorithm}

Recall that $\Delta(x,y)$ is the Hamming distance between binary strings $x,y$.
Algorithm~\ref{alg:amp-opti-general} shows the fully general algorithm computing
$\Post$ given $d,k,c,\varepsilon$, and $\alpha$. We prove the algorithm computes
\begin{equation*}
    \sup_{\substack{ P,Q \in \calC_d \\ \renyi(P,Q) \leq \varepsilon \\ \renyi(Q,P)
    \leq \varepsilon}} \renyi(B_k(P), B_k(Q))
\end{equation*}
which is equal to $\Postk$, by Theorem~\ref{thm:multi-bounded-amp-ms}.
Let $z_i$ for $i = (i_1, \ldots, i_d) \in \{0,1\}^d$ be the point $\{c^{i_1}
(1-c)^{1-i_i}, \ldots, c^{i_d}(1-c)^{1-i_d}\}$.
Let $\{x_i : i \in \{0,1\}^d \}$ represent the mass that $P$ places at $z_i$, and let
  $\{ y_i : i \in \{0,1\}^d \}$ be defined similarly for $Q$. Plugging 
directly into the
definition of definition of R\'enyi divergence, $Constraint_1$ and
$Constraint_2$ represent the constraints $\renyi(P,Q) \leq \varepsilon$ and
$\renyi(Q,P) \leq \varepsilon$. The function
\[
  \sum_{i \in \{0,1\}^d} \left(\frac{x_i}{y_i}\right)^\alpha y_i
\]
is an $f$-divergence, and is therefore a convex function~\cite{Cichocki:2010}.
The constraints that the $x_i$ and $y_i$ form two probability distributions are
obvious and are included in the last line of the algorithm. The only step we
have left is to express the objective function $\renyi(B_k(P), B_k(Q))$.
For $b_1, \ldots,b_k$, each in $\{0,1\}^d$ we have
\begin{align*}
  &\Pr[B_k(P) = (b_1, \ldots, b_k)] \\ &\qquad = \sum_{i \in \{0,1\}^d} \Pr[P = z_i]
  \Pr[B_k(z_i) = (b_1, \ldots, b_k)] \\
  &\qquad = \sum_{i \in \{0,1\}^d } x_i\prod_{j=1}^k \Pr[B(z_i) = b_j] \\
  &\qquad = \sum_{i \in \{0,1\}^d } x_i
  \prod_{j=1}^k c^{\Delta(i, b_j)}(1-c)^{d-\Delta(i, b_j)}
\end{align*}
where the first equality follows from conditioning on the value of $P$, the
second from independence of the $k$ samples, and the third from independence of
the $d$ coordinates.
A similar formula holds for $Q$.
It suffices to optimize, with $\textbf{b} = (b_1, \ldots, b_k)$,
\begin{multline}\label{eq:renyi-opti}
  e^{(\alpha-1)\renyi(B_k(P), B_k(Q))} = \\ 
  \sum_{b_1, \ldots, b_k \in \{0,1\}^d}
  \Pr[B_k(P) = \textbf{b}]^\alpha \Pr[B_k(Q) = \textbf{b}]^{1-\alpha}
\end{multline}
The algorithm has these equations exactly, but it uses the variables 
\begin{align*}
\barx_{b_1, \ldots, b_k} &=
\Pr[B_k(P) = (b_1, \ldots, b_k)] \\
\bary_{b_1, \ldots, b_k} &= \Pr[B_k(Q) = (b_1, \ldots, b_k)]
\end{align*}
The objective function is convex in the variables $\barx_{b_1, \ldots, b_k}$ and
$\bary_{b_1, \ldots, b_k}$, and each $\barx_{b_1, \ldots, b_k}$ and $\bary_{b_1,
\ldots, b_k}$ is an affine transformation of the input variables $x_i$ and
$y_i$. Therefore, Algorithm~\ref{alg:amp-opti-general} is a convex optimization problem.
The runtime bottleneck is in evaluating the objective function which takes 
$2^{O(dk)}$ time. 

Algorithm~\ref{alg:amp-opti} follows from specializing
Algorithm~\ref{alg:amp-opti-general} to $k=1$.
When $d=1$, there is more simplification. For $\textbf{b} = (b_1, \ldots, b_k)$
in $\{0,1\}^k$,
\begin{align*}
  \Pr[B_k(P) = \textbf{b}] 
  &= \sum_{i \in \{0,1\} } x_i
  \prod_{j=1}^k c^{\Delta(i, b_j)}(1-c)^{1-\Delta(i, b_j)} \\
  &= x_0 c^{\sum b_j} (1-c)^{k-\sum b_j} + x_1 c^{k-\sum b_j} (1-c)^{\sum b_j}
\end{align*}
This only depends on the number of $0s$ in $(b_1, \ldots, b_k)$.
Summing~(\ref{eq:renyi-opti}) over the number of possible zeros in $b_1, \ldots,
b_k$, we get
\begin{multline*}
  e^{(\alpha-1)\renyi(B_k(P), B_k(Q))} = \\
  \sum_{j=0}^k { {k} \choose {j} } \big( (x_0 c^j (1-c)^{k-j} + x_1 c^{k-j}
  (1-c)^j)^\alpha \\ (y_0 c^{j} (1-c)^{k-j} + y_1 c^{k-j}
(1-c)^{j})^{1-\alpha}\big)
\end{multline*}

Algorithm~\ref{alg:amp-opti-multi} follows by using the variable $(x, 1-x)$ in place
of $x_0,x_1$ and changing the other constraints slightly.
\section{Proofs}
\label{app:proofs}

\begin{proof}
  (Of Theorem~\ref{thm:multi-bounded-amp}, \ref{thm:multi-bounded-amp-ms})
  Recall that
  \[
    \Postk = \sup_{A \in \calA, \varepsilon_A(\alpha) \leq \varepsilon}
    \varepsilon_{B_k(A)}(\alpha)
  \]
  For a fixed $\calA$ such that $\varepsilon_A(\alpha) \leq \varepsilon$, we
  have, where $\supp(P)$ is the support of distribution $P$,
  \[
    \varepsilon_{B_k(A)}(\alpha) \leq \sup_{\substack{ \supp(P), \supp(Q)
    \subseteq \Theta \\ \renyi(P,Q) \leq \varepsilon \\ \renyi(Q,P) \leq
    \varepsilon }} \renyi(B_k(P), B_k(Q))
  \]
  This is because an $A \in \calA$ always releases $P$ and $Q$ supported on $\Theta$ such
  that $\renyi(P,Q)$ and $\renyi(Q,P)$ are less than $\varepsilon$, meaning that
  $\varepsilon_{B_k(A)}(\alpha)$ is less than the $\sup$ above. 

  A weakening of Lemma~\ref{lem:discretization} tells us that for any $P$,$Q$,
  we can actually find $c$-corner distributions $P',Q'$ such that
  $\renyi(P',Q') \leq \renyi(P,Q)$, $\renyi(Q',P') \leq \renyi(Q,P)$, and
  $\renyi(B_k(P'), B_k(Q')) = \renyi(B_k(P), B(Q))$. Therefore,
  \begin{multline*}
    \sup_{\substack{ \supp(P), \supp(Q)
    \subseteq \Theta \\ \renyi(P,Q) \leq \varepsilon \\ \renyi(Q,P) \leq
    \varepsilon }} \renyi(B_k(P), B_k(Q)) = \\
    \sup_{\substack{ P,Q \in \calC_d 
    \\ \renyi(P,Q) \leq \varepsilon \\ \renyi(Q,P) \leq
    \varepsilon }} \renyi(B_k(P), B_k(Q))
  \end{multline*}
  
  This means $\varepsilon_{B_k(A)}(\alpha)$ is upper bounded by the $\sup$ 
  above for all $A \in \calA$. However, there is an $A \in \calA$ for which
  equality holds: If $P,Q \in \calC_d$ maximize the $\sup$ above, then $A$ just
  releases $A(D) = P$ and $A(D') = Q$ for two neighboring databases $D,D'$.
\end{proof}

\begin{lemma}\label{lem:discretization}
  Let the space $\Theta = \prod_{i=1}^d (c_i, d_i)$ for $0 < c_i \leq d_i < 1$.
  Let $\Delta = \prod_{i=1}^d \{c_i, d_i\}$. Let $P,Q$ be distributions on
  $\Theta$. Then, there exist distributions $P', Q'$ on $\Delta$ such that 
  $\renyi(P', Q') \leq \renyi(P,Q)$, $\renyi(Q', P') \leq \renyi(Q, P)$, and
  $\renyi(B_k(P), B_k(Q)) = \renyi(B_k(P'), B_k(Q'))$.
\end{lemma}
\begin{proof} 

  We will prove this theorem on arbitrarily fine discretizations of $P$ and $Q$.
  We now assume $P$ and $Q$ are discrete.
  Suppose $P,Q$ place mass on a point 
  $x \notin \Delta$. By Lemma~\ref{lem:extremal-bern}, we can write, for
  coefficients $\{\lambda_z\}_{z \in \Delta}$ of a convex combination,
  \[
    B_k(x) = \sum_{z \in \Delta} \lambda_z B_k(z)
  \]
  For any $x \in \Theta$, we have $B_k(x) = B_k(\indicator[X=x])$. Thus,
  \[
    B_k(\indicator[X=x]) = \sum_{z \in \Delta} \lambda_z B_k(\indicator[X=z])
  \]
  Notice $B_k$ is a Markov operator, so it factors across sums:
  \[
    \sum_{z \in \Delta} \lambda_z B_k(\indicator[X=z])
    = B_k\left( \sum_{z \in \Delta} \lambda_z \indicator[X=z] \right)
    = B_k(\textbf{m})
  \]
  where $\textbf{m}$ is the probability distribution that takes value $z \in \Delta$ 
  w.p. $\lambda_z$. Thus, we conclude $B_k(\textbf{m}) =
  B_k(\indicator[X=x]) = B_k(x)$.
  Let 
  \begin{align*}
    P' &= P - \Pr[P=x] \indicator[X=x] + \Pr[P=x] \textbf{m} \\
    Q' &= Q - \Pr[Q=x] \indicator[X=x] + \Pr[Q=x] \textbf{m}
  \end{align*}
  Then, using the Markov property of $B_k$ again,
  \begin{align*}
    B_k(P') &= B_k(P) - \Pr[P=x] B_k(\indicator[X=x]) 
    + \Pr[P=x] B_k(\textbf{m})\\
    B_k(Q') &= B_k(Q) - \Pr[Q=x] B_k(\indicator[X=x]) + \Pr[Q=x] 
    B_k(\textbf{m})
  \end{align*}
  Because $B_k(\textbf{m}) = B_k(\indicator[X=x])$, the above equations
  simplify to $B_k(P') = B_k(P)$ and $B_k(Q') = B_k(Q)$.
  Hence, $\renyi(B_k(P), B_k(Q)) = \renyi(B_k(P'),B_k(Q'))$. 
  However, $P', Q'$ are
  a post-processing of $P,Q$: if we observe $X=x$, then we sample from
  $\textbf{m}$ instead. By the data processing inequality, $\renyi(P',Q') \leq 
  \renyi(P,Q)$ and $\renyi(Q', P') \leq \renyi(Q, P)$.
\end{proof}

\begin{lemma}\label{lem:extremal-bern}
  Let $\Theta,\Delta$ be defined as in Lemma~\ref{lem:discretization}.
  For a point $x \in \Theta$, we can write $B_k(x)$ as the
  convex combination $ \sum_{z \in \Delta} \lambda_z B_k(z)$
  for coefficients $\{\lambda_z : z \in \Delta\}$.
\end{lemma}
\begin{proof}
  $B_k(z)$ has the same distribution on each of its $k$ samples for any
  $z \in \Theta$. Thus, it suffices to prove the theorem for $k=1$.
  Let $A \otimes B$ be the product distribution of $A$ and $B$. By definition,
  \[
    B_1(x) = \bigotimes_{i=1}^d B(x_i)
  \]
  Because $c_i \leq x_i \leq d_i$, there is some
  $\lambda_i$ such that $B(x_i) = \lambda_i B(c_i) + (1-\lambda_i)
  B(d_i)$. Therefore, we can write
  \begin{align*}
    B_k(x) &= \bigotimes_{i=1}^d \lambda_i B(c_i) +
    (1-\lambda_i) B(d_i)
  \end{align*}
  It is well known that the product of two convex combinations of distribution
  is itself a convex distribution. Each term of this convex combination will be
  $B(z_1) \otimes \cdots \otimes B(z_d)$ for $z_i \in \{c_i, d_i\}$.
  This is equal to $B( (z_1, \ldots, z_d) )$, and $(z_1, \ldots, z_d) \in
  \Delta$.
\end{proof}

\begin{proof}
  (Of Theorem~\ref{thm:multi-lower-bound}):
  Recall $P = p \indicator[X=\{c\}^d + (1-p) \indicator[X = \{1-c\}^d]$ and
  $Q = (1-p) \indicator[X=\{c\}^d + p \indicator[X = \{1-c\}^d]$.
  $P,Q$ and $B(P)$, $B(Q)$ are isomorphic pairs of distributions under
  flipping by flipping the 0s and 1s in their domains. Thus, $\renyi(P,Q) =
  \renyi(Q,P)$ and $\renyi(B(P), B(Q)) = \renyi(B(Q), B(P))$.

  We let $\#_0(x)$ be the number of 0s appearing in binary vector $x$ and $\barp
  = 1-p$.
  We write the following lower and upper bound on the probability that 
  $B(P)$ has many 1s: 
  \begin{align*}
    &\Pr[\#_1( B(P) ) \geq d/2] \\
    &\;\geq
    \Pr[P = \{1-c\}^d]\Pr[\#_1(B(\{1-c\}^d)) \geq d/2] \\
    &\;\geq \barp (1-\Pr[\#_1(B(\{1-c\}^d)) < d/2]) \\
    &\Pr[\#_1( B(P) ) \geq d/2] \\
    &\;\leq
    \Pr[P = \{1-c\}^d] + \Pr[P = \{c\}^d]\Pr[\#_1(B(\{c\}^d)) > d/2] \\
    &\;\leq \barp + p\Pr[\#_1(B(\{c\}^d)) > d/2]
  \end{align*}
  Hoeffding's inequality tells us that
  \begin{align*}
    \Pr[\#_1(B(\{1-c\}^d)) < d/2] &\leq e^{-2(1/2-c)^2d} \\
    \Pr[\#_1(B(\{c\}^d)) > d/2] &\leq e^{-2(1/2-c)^2d}
  \end{align*}
  Therefore, with $K = e^{-2(1/2-c)^2d}$,
  \[
    \barp(1-K) \leq \Pr[\#_1(B(P)) \geq d/2] \leq \barp + pK
  \]
  Let $\maj$ be the majority function, so that
  \begin{align*}
    \maj(B(P)) &= \Pr[\#_1(B(P)) \geq d/2] \indicator[X = 1] \\
    &\qquad + 
    (1-\Pr[\#_1(B(P)) \geq d/2]) \indicator[X = 0] \\
    &\qquad = B(\Pr[\#_1(B(P)) \geq d/2])
  \end{align*}
  
  From our above 
  analysis, $\maj(B(P)) = B(p')$ for some 
  $p' \in [\barp(1-K), \barp + p K]$. Relax the interval to $p' \in
  \barp \pm K$. Because $B(Q)$ is isomorphic
  to $B(P)$ by flipping the 0s and 1s in the domain, we also have 
  $\maj(B(Q)) = B(1-p')$. Therefore,
  \[
    \renyi(B(P), B(Q)) \geq \renyi(
    \maj(B(P)), \maj(B(Q))) = r_\alpha(p')
  \]
  Since $r_\alpha(p)$ has one
  local minimum at $p = \frac{1}{2}$, the minimum value it
  achieves on the interval $[a,b]$ assuming $\frac{1}{2} \leq a \leq b$ is 
  is achieved at $a$. The lower edge of the interval $\barp \pm K$ is bigger
  than $\frac{1}{2}$ since $p + K < \frac{1}{2}$. Thus,
  \[
    r_\alpha(p') \geq \inf_{x \in \barp + [-K, K]}
    r_\alpha(x) \geq r_\alpha(\barp - K) = r_\alpha(p + K)
  \]
\end{proof}

\begin{theorem}\label{thm:amp-lb}
  Let 
  \begin{align*} 
    P &= p \indicator[X=\{c\}^d] + (1-p)\indicator[X=\{1-c\}^d] \\
    Q &= (1-p) \indicator[X=\{c\}^d] + p \indicator[X=\{1-c\}^d]
  \end{align*}
  be distributions. Let $P_j = pc^j (1-c)^{dk-j} + (1-p) c^{dk-j} (1-c)^{j}$ 
  and $Q_j = (1-p)c^j (1-c)^{dk-j} + p c^{dk-j} (1-c)^j$. Then,
  \begin{equation*}
    \renyi(B_k(P), B_k(Q)) = \\
    \frac{1}{\alpha-1} \log \left(
      \sum_{j = 0}^{dk} { {dk} \choose {j} } P_j^\alpha Q_j ^{1-\alpha}
  \right)
  \end{equation*}
\end{theorem}
\begin{proof}
  We have, for $\textbf{x} = (x_1, \ldots, x_k)$ each in $\{0,1\}^d$,
  \begin{align*}
    \Pr[B_k(P) = \textbf{x}] &= p\Pr[B_k(\{c\}^d) = \textbf{x}
    ] \\ &\qquad + (1-p) \Pr[B_k(\{1-c\}^d) = \textbf{x}]\\
    &= p c^{\sum \#_1(x_i)} (1-c)^{dk - \sum \#_1(x_i)} \\
    &\qquad + (1-p) c^{dk - \sum \#_1(x_i)} (1-c)^{\sum \#_1(x_i)}
  \end{align*}
  Where the first equality follows from conditioning on $P$ and the second from
  independence of $B_k$ across coordinates and of $B(\{c\}^d),
  B(\{1-c\}^d)$ across their $d$ dimensions. Thus, the mass of
  $B_k(P)$ at a point in $\{0,1\}^{dk}$ with Hamming weight $j$ is 
  $P_j$. For $B(Q)$, it is $Q_j$. There are ${ {dk} \choose {j} }$
  points in $\{0,1\}^{dk}$ with Hamming weight $j$, so the result follows by the 
  definition of R\'enyi divergence.
\end{proof}

\begin{theorem}
  \label{thm:amp-asymptote-multi}
  $\Postk \leq R(B_k\{c\}^d, B_k(\{1-c\}^d)) = dkr_\alpha(c)$.
\end{theorem}
\begin{proof}
  We get rid of some constraints of the $\sup$ of
  Theorem~\ref{thm:multi-bounded-amp}, getting
  \[
    \Postk_{\calA, \alpha}(\varepsilon) \leq \sup_{P,Q \in \calC_d} 
    \renyi(B_k(P), B_k(Q))
  \]
  A general $P,Q \in \calC_d$ can be written as
  \begin{align*}
    P &= \sum_{i \in \{c,1-c\}^d} p_i \indicator[X=i] \\
    Q &= \sum_{i \in \{c,1-c\}^d} q_i \indicator[X=i]
  \end{align*}
  By the Bernoulli process is a Markov process, applying $B$ to both sides
  gives
  \begin{align*}
    B_k(P) &= \sum_{i \in \{c,1-c\}^d} p_i B_k(i) \\
    B_k(Q) &= \sum_{i \in \{c,1-c\}^d} q_i B_k(i)
  \end{align*}
  Quasi-convexity of the R\'enyi Divergence~\cite{erven:2012} states that 
  for distributions $P_1, \ldots, P_n$ and $Q_1, \ldots, Q_n$, and a convex
  combination $\lambda_1, \ldots, \lambda_n$,
  \begin{equation*}
    \renyi(\lambda_1 P_1 + \cdots + \lambda_n P_n, \lambda_1 Q_1 + \cdots +
    \lambda_n Q_n) \leq
    \max_{i=1}^n \renyi(P_i, Q_i)
  \end{equation*}
  Here, our convex combinations $p_1, \ldots, p_n$ and $q_1, \ldots, q_n$ are
  different, but we can pair them up as follows: if $p_i$ the smallest nonzero
  coefficient out of all $p_i$s and $q_i$s, then pair it with an arbitrary
  nonzero $q_j$. Set $p_i = 0$ and $q_j = q_j - p_i$. This process will
  terminate eventually, and we will be left with the equality
  \[
    \left( \sum_{i \in \{c,1-c\}^d }p_i B_k(i), \sum_{i \in \{c, 1-c\}^d} q_i
    B_k(i)  \right) = \sum_{i \in \{c, 1-c\}^d }\lambda_x
      (B_k(i),B_k(i)
  \]
  By quasi-convexity,
  \[
    \renyi(B_k(P), B_k(Q)) \leq \max_{i,j \in \{c,1-c\}^d}
    \renyi(B_k(i), B_k(j))
  \]
  Because $B_k(i)$ are independent across the $d$ coordinates of $i$, 
  for $i = (i_1, \ldots, i_d)$ and $j = (j_1, \ldots, j_d)$, and by the
  additive property of R\'enyi divergence across product
  distributions~\cite{erven:2012},
  \[
    \renyi(B_k(i), B_k(j)) = \sum_{\ell=1}^d \renyi(B_k(i_\ell),
    B_k(j_\ell))
  \]
  Each $\renyi(B_k(i_\ell), B_k(j_\ell))$ is zero when $i_k = j_k$ 
  and equal to $kr_\alpha(c)$ otherwise, again using the additive property of
  R\'enyi divergence across the $k$ product distributions. Therefore,
  \[
    \sum_{\ell=1}^d \renyi(B_k(i_\ell), B_k(j_\ell)) \leq d kr_\alpha(c)
  \]
\end{proof}

\end{document}